\documentclass[12pt]{article}
\usepackage{geometry,mathtools}
\usepackage{amsmath,amssymb,amsthm, amsfonts}
\usepackage{etaremune, enumerate, float, verbatim}
\usepackage{algorithm}
\usepackage{algorithmic}
\usepackage{hyperref}
\usepackage{graphicx}
\usepackage{url}
\usepackage[mathlines]{lineno}
\usepackage{dsfont} 
\usepackage{tikz, graphicx, subcaption, caption}
\usepackage[algo2e,ruled,vlined]{algorithm2e}
\usepackage{mathrsfs,amsmath}
\usepackage{color}
\definecolor{red}{rgb}{1,0,0}

\definecolor{blue}{rgb}{0,0,.7}

\definecolor{green}{rgb}{0,.6,0}

\definecolor{purp}{rgb}{.5,0,.5}

\numberwithin{figure}{section}   

\newtheorem{thm}{Theorem}[section]
\newtheorem{cor}[thm]{Corollary}
\newtheorem{lem}[thm]{Lemma}

\theoremstyle{definition}

\theoremstyle{definition}

\theoremstyle{definition}

\newcommand{\opt}{\operatorname{opt}}

\newcommand{\linint}{\operatorname{LININT}}

\newcommand{\bit}{\begin{itemize}}
\newcommand{\eit}{\end{itemize}}
\newcommand{\ben}{\begin{enumerate}}
\newcommand{\een}{\end{enumerate}}
\newcommand{\beq}{\begin{equation}}
\newcommand{\eeq}{\end{equation}}
\newcommand{\bea}{\begin{eqnarray*}} 
\newcommand{\eea}{\end{eqnarray*}}
\newcommand{\bpf}{\begin{proof}}
\newcommand{\epf}{\end{proof}\ms}
\newcommand{\bmt}{\begin{bmatrix}}
\newcommand{\emt}{\end{bmatrix}}
\newcommand{\ms}{\medskip}

\newcommand{\noi}{\noindent}

\title{Sharper bounds for online learning of smooth functions of a single variable}
\author{Jesse Geneson}

\begin{document}
\maketitle

\begin{abstract}
We investigate the generalization of the mistake-bound model to continuous real-valued single variable functions. Let $\mathcal{F}_q$ be the class of absolutely continuous functions $f: [0, 1] \rightarrow \mathbb{R}$ with $||f'||_q \le 1$, and define $\opt_p(\mathcal{F}_q)$ as the best possible bound on the worst-case sum of the $p^{th}$ powers of the absolute prediction errors over any number of trials. Kimber and Long (Theoretical Computer Science, 1995) proved for $q \ge 2$ that $\opt_p(\mathcal{F}_q) = 1$ when $p \ge 2$ and $\opt_p(\mathcal{F}_q) = \infty$ when $p = 1$. For $1 < p < 2$ with $p = 1+\epsilon$, the only known bound was $\opt_p(\mathcal{F}_{q}) = O(\epsilon^{-1})$ from the same paper. We show for all $\epsilon \in (0, 1)$ and $q \ge 2$ that $\opt_{1+\epsilon}(\mathcal{F}_q) = \Theta(\epsilon^{-\frac{1}{2}})$, where the constants in the bound do not depend on $q$. We also show that $\opt_{1+\epsilon}(\mathcal{F}_{\infty}) = \Theta(\epsilon^{-\frac{1}{2}})$. 
\end{abstract}

\noi {\bf Keywords}: online learning, smooth functions, single variable, general loss functions

\section{Introduction}

We consider a model of online learning previously studied in \cite{angluin, ls, lw, myc, klsf, smooth2} where an algorithm $A$ tries to learn a real-valued function $f$ from some class $\mathcal{F}$. In each trial of this model, $A$ receives an input $x_t \in [0, 1]$, it must output some prediction $\hat{y}_t$ for $f(x_t)$, and then $A$ discovers the true value of $f(x_t)$. 

For each $p \ge 1$ and $x = (x_0, \dots, x_m) \in [0, 1]^{m+1}$, define $L_p(A, f, x) = \sum_{t = 1}^m |\hat{y}_t-f(x_t)|^p$. Note that the summation starts on the second trial, since the guess on the first trial does not reflect the algorithm's learning ability. Define $L_p(A, \mathcal{F}) = \displaystyle \sup_{f \in \mathcal{F}, x \in  \cup_{m \in \mathbb{N}} [0,1]^{m+1}} L_p(A, f, x)$. In this paper, we study the optimum $\opt_p(\mathcal{F}) = \displaystyle \inf_A L_p(A, \mathcal{F})$. 

In particular we focus on the class of functions whose first derivatives have $q$-norms at most $1$. For any real number $q \ge 1$, let $\mathcal{F}_q$ be the family of absolutely continuous functions $f: [0, 1] \rightarrow \mathbb{R}$ for which $\int_{0}^1 |f'(x)|^q dx \le 1$. Let $\mathcal{F}_{\infty}$ be the family of absolutely continuous functions $f:  [0, 1] \rightarrow \mathbb{R}$ for which $\displaystyle \sup_{x \in [0, 1]} |f'(x)| \le 1$. By the definition of $\mathcal{F}_{\infty}$ and Jensen's inequality respectively, we have $\mathcal{F}_{\infty} \subseteq \mathcal{F}_r \subseteq \mathcal{F}_q$ for any $1 \le q \le r$. Thus $\opt_p(\mathcal{F}_{\infty}) \le \opt_p(\mathcal{F}_r) \le \opt_p(\mathcal{F}_q)$ for any $1 \le q \le r$.

Kimber and Long \cite{klsf} proved that $\opt_p(\mathcal{F}_1) = \infty$ for all $p \ge 1$. They also showed that $\opt_1(\mathcal{F}_q) = \opt_1(\mathcal{F}_{\infty}) = \infty$ for all $q \ge 1$. In contrast, they found that $\opt_p(\mathcal{F}_q) = \opt_p(\mathcal{F}_{\infty}) = 1$ for all $p \ge 2$ and $q \ge 2$. This was also proved by Faber and Mycielski \cite{FM} using a different algorithm, and a noisy version of this problem was studied in \cite{clw}.

For $p = 1+\epsilon$ with $\epsilon \in (0, 1)$, Kimber and Long proved that $\opt_p(\mathcal{F}_q) = O(\epsilon^{-1})$ for all $q \ge 2$, which implies that $\opt_p(\mathcal{F}_{\infty}) = O(\epsilon^{-1})$. However, these bounds are not sharp. In this paper, we determine $\opt_{1+\epsilon}(\mathcal{F}_q)$ up to a constant factor for all $\epsilon \in (0, 1)$ and $q \ge 2$. 

\begin{thm}\label{mainth}
For all $\epsilon \in (0, 1)$, we have $\opt_{1+\epsilon}(\mathcal{F}_{\infty}) = \Theta(\epsilon^{-\frac{1}{2}})$ and $\opt_{1+\epsilon}(\mathcal{F}_q) = \Theta(\epsilon^{-\frac{1}{2}})$ for all $q \ge 2$, where the constants in the bound do not depend on $q$.
\end{thm}

The proof splits into an upper bound and a lower bound. For the upper bound, we use H\"older's inequality combined with past results of Kimber and Long. For the lower bound, we modify a construction used in \cite{smooth2}, where Long obtained bounds on a finite variant of $\opt_1(\mathcal{F}_q)$ that depends on the number of trials $m$.

\section{Sharp bounds for general loss functions}

In this section, we derive an upper bound and a lower bound to prove Theorem \ref{mainth}. We introduce some terminology from \cite {klsf} and \cite{smooth2} that we use for the proofs in this section. For a function $f: [0, 1] \rightarrow \mathbb{R}$, we define $J[f] = \int_{0}^1 f'(x)^2 dx$. Given a finite subset $S \subseteq [0, 1] \times \mathbb{R}$ with $S = \left\{(u_i, v_i): 1 \le i \le m\right\}$ and $u_1 < u_2 < \cdots < u_m$, we define $f_S: [0, 1] \rightarrow \mathbb{R}$ as follows. Let $f_{\emptyset}(x) = 0$ for all $x$, and for each nonempty $S$ let $f_S$ be the piecewise function defined by $f_S(x) = v_1$ for $x \le u_1$, $f_S(x) = v_i+\frac{(x-u_i)(v_{i+1}-v_i)}{u_{i+1}-u_i}$ for $x \in (u_i, u_{i+1}]$, and $f_S(x) = v_m$ for $x > u_m$. 

For the upper bound, we use the $\linint$ learning algorithm which is defined using $f_S$. Specifically we define $\linint(\emptyset, x_1) = 0$ and $\linint(((x_1, y_1), \dots, (x_{t-1},y_{t-1}),x_t) = f_{\left\{(x_1, y_1), \dots, (x_{t-1},y_{t-1}) \right\}}(x_t)$. The beginning of our proof is similar to Kimber and Long's proof from \cite{klsf} that $\opt_{1+\epsilon}(\mathcal{F}_2) = O(\epsilon^{-1})$ for $\epsilon \in (0, 1)$, but we change the end of the proof by using H\"older's inequality to obtain a sharper bound of $\opt_{1+\epsilon}(\mathcal{F}_2) = O(\epsilon^{-\frac{1}{2}})$.

\begin{thm}\label{upperbound}
If $\epsilon \in (0, 1)$, then $\opt_{1+\epsilon}(\mathcal{F}_2) = O(\epsilon^{-\frac{1}{2}})$.
\end{thm}

\begin{proof}
Fix $\epsilon \in  (0, 1)$ and let $p = 1+\epsilon$. Let $x_0, \dots, x_m$ be any sequence of elements of $[0, 1]$, and let $f \in \mathcal{F}_2$. Let $\hat{y}_1, \dots, \hat{y}_m$ be $\linint$'s predictions on trials $1, \dots, m$. For each $i > 1$, let $d_i = \min_{j < i} |x_j - x_i|$ and let $e_i = |\hat{y}_i - f(x_i)|$. Kimber and Long proved in Theorem 17 of \cite{klsf} that $\sum_{i = 1}^m \frac{e_i^2}{d_i} \le 1$. In the same theorem, they also proved that $\sum_{i = 1}^m d_i^p \le 1 + \frac{1}{2^p - 2}$ for $p > 1$. This is where our proofs diverge. 

First, note that $\sum_{i = 1}^m e_i^p = \sum_{i = 1}^m \frac{e_i^p}{d_i^{\frac{p}{2}}} \cdot d_i^{\frac{p}{2}}$. By H\"older's inequality, we have 

\begin{align*}
 \sum_{i = 1}^m \frac{e_i^p}{d_i^{\frac{p}{2}}} \cdot d_i^{\frac{p}{2}} \le \\
(\sum_{i = 1}^m (\frac{e_i^p}{d_i^{\frac{p}{2}}})^{\frac{2}{p}})^{\frac{p}{2}} ( \sum_{i = 1}^m  (d_i^{\frac{p}{2}})^{\frac{2}{2-p}})^{1 - \frac{p}{2}}.
\end{align*}

Note that $\sum_{i = 1}^m (\frac{e_i^p}{d_i^{\frac{p}{2}}})^{\frac{2}{p}} = \sum_{i = 1}^m  \frac{e_i^2}{d_i} \le 1$ by the first bound from \cite{klsf} that we cited in the first paragraph. By the second bound from \cite{klsf} that we cited in the first paragraph, we have $\sum_{i = 1}^m  (d_i^{\frac{p}{2}})^{\frac{2}{2-p}} = \sum_{i = 1}^m  d_i^{\frac{p}{2-p}}  \le 1 + \frac{1}{2^{\frac{p}{2-p}}-2}$ since $\frac{p}{2-p} > 1$.

Thus $( \sum_{i = 1}^m  (d_i^{\frac{p}{2}})^{\frac{2}{2-p}})^{1 - \frac{p}{2}} \le (1 + \frac{1}{2^{\frac{p}{2-p}}-2})^{1-\frac{p}{2}} =  (1 + \frac{1}{2^{\frac{1+\epsilon}{1-\epsilon}}-2})^{\frac{1-\epsilon}{2}}$.

Let $\delta = \frac{1+\epsilon}{1-\epsilon}-1$, and note that $\delta = \frac{2\epsilon}{1-\epsilon} \ge 2\epsilon$ since $\epsilon \in (0, 1)$. Thus $(1 + \frac{1}{2^{\frac{1+\epsilon}{1-\epsilon}}-2})^{\frac{1-\epsilon}{2}} = (1 + \frac{1}{2^{1+\delta}-2})^{\frac{1-\epsilon}{2}} = O(\delta^{-\frac{1-\epsilon}{2}})$ since $e^{\delta \ln{2}} \ge 1+\delta \ln{2}$. Moreover $\delta^{-\frac{1-\epsilon}{2}} = O(\epsilon^{-\frac{1-\epsilon}{2}}) = O(\epsilon^{-\frac{1}{2}})$ since $\delta \ge 2\epsilon$ and $\epsilon^\epsilon = \Theta(1)$ for $\epsilon \in (0, 1)$.

Thus we have proved that $\sum_{i = 1}^m e_i^p  = O(\epsilon^{-\frac{1}{2}})$, so $\opt_{1+\epsilon}(\mathcal{F}_2) = O(\epsilon^{-\frac{1}{2}})$.
\end{proof}

We obtain the next corollary from Theorem \ref{upperbound} since $\opt_p(\mathcal{F}_{\infty}) \le \opt_p(\mathcal{F}_r) \le \opt_p(\mathcal{F}_q)$ whenever $1 \le q \le r$.

\begin{cor}
If $\epsilon \in (0, 1)$, then $\opt_{1+\epsilon}(\mathcal{F}_{\infty}) = O(\epsilon^{-\frac{1}{2}})$ and $\opt_{1+\epsilon}(\mathcal{F}_q) = O(\epsilon^{-\frac{1}{2}})$ for all $q \ge 2$, where the constant does not depend on $q$.
\end{cor}

In order to show that the last theorem is sharp up to a constant factor, we construct a family of functions in $\mathcal{F}_{\infty}$. Our proof uses the following lemma from \cite{klsf} which was also used in \cite{smooth2}.

\begin{lem}\label{kl_jlem}
Let $S \subseteq [0, 1] \times \mathbb{R}$ with $S = \left\{(u_i, v_i): 1 \le i \le m\right\}$ and $u_1 < u_2 < \cdots < u_m$. If $(x,y) \in [0, 1] \times \mathbb{R}$ and there exists $1 \le j \le m$ such that $|x-u_j| = |x-u_{j+1}| = \min_i |x-u_i|$, then $J[f_{S \cup \left\{(x,y) \right\}}] = J[f_S] + \frac{2(y-f_S(x))^2}{\min_i |x-u_i|}$.
\end{lem}

The construction and method in the following proof is very similar to one used by Long in \cite{smooth2} to obtain bounds for a finite variant of $\opt_1(\mathcal{F}_q)$ for $q \ge 2$ that depends on the number of trials $m$. The proofs differ in that we have adjusted the parameters of the construction to give the desired lower bound in terms of $\epsilon$, and some summations that were finite in Long's proof are infinite in the next proof.

\begin{thm}\label{lowerbound}
If $\epsilon \in (0, 1)$, then $\opt_{1+\epsilon}(\mathcal{F}_{\infty}) = \Omega(\epsilon^{-\frac{1}{2}})$.
\end{thm}

\begin{proof}
Since $\opt_{1+\epsilon}(\mathcal{F}_{\infty}) \ge 1$ for all $\epsilon \in (0, 1)$, it suffices to prove the theorem for $\epsilon \in (0, \frac{1}{2})$. Define $x_0 = 1$ and $y_0 = 0$. For natural numbers $i, j$ with $0 \le j < 2^{i-1}$, define $x_{2^{i-1}+j} = \frac{1}{2^i}+\frac{j}{2^{i-1}}$. For each $i = 1, 2, \dots$, we consider the trials for $x_{2^{i-1}}, \dots, x_{2^i - 1}$ to be part of stage $i$, so that $x_1 = \frac{1}{2}$ is in stage $1$, $x_2 = \frac{1}{4}$ and $x_3 = \frac{3}{4}$ are in stage $2$, and so on.

Let $A$ be any algorithm for learning $\mathcal{F}_{\infty}$. Using $A$, we construct an infinite sequence of functions $f_0, f_1, \dots \in \mathcal{F}_{\infty}$ and an infinite sequence of real numbers $y_0, y_1, \dots$ for which $f_i$ is consistent with the $x_k$ and $y_k$ values for $k \le i$ and $A$ has total $(1+\epsilon)$-error at least $\sum_{k = 1}^i 2^{k-2} (\frac{\sqrt{\epsilon}(1-\epsilon)^{\frac{k}{2}}}{2^{k+1}})^{1+\epsilon}$ when $f_i$ is the target function. This will imply that $\opt_{1+\epsilon}(\mathcal{F}_{\infty}) \ge \sum_{k = 1}^{\infty} 2^{k-2} (\frac{\sqrt{\epsilon}(1-\epsilon)^{\frac{k}{2}}}{2^{k+1}})^{1+\epsilon}$. 

For the proof, we will also define another infinite sequence of functions $g_{i, j}$ with $0 \le j \le 2^{i-1}$ and another infinite sequence of real numbers $v_1, v_2, \dots$. We start by letting $f_0$ be the $0$-function.

Fix a stage $i$, and let $g_{i, 0} = f_{2^{i-1}-1}$. Let $t$ be a trial in stage $i$, and let $v_t$ be whichever of $f_{t-1}(x_t) \pm \frac{\sqrt{\epsilon}(1-\epsilon)^{\frac{i}{2}}}{2^{i+1}}$ is furthest from $\hat{y}_t$. Let $g_{i, t-2^{i-1}+1}$ be the function which linearly interpolates $\left\{(0, 0), (1, 0) \right\} \cup \left\{(x_s, y_s): s < 2^{i-1} \right\} \cup \left\{(x_s, v_s): 2^{i-1} \le s \le t \right\}$. 

For any $t \ge 1$, let $L_t$ and $R_t$ be the elements of $\left\{0, 1 \right\} \cup \left\{x_s: s < t \right\}$ that are closest to $x_t$ on the left and right respectively. If both $|v_t - f_{t-1}(L_t)| \le 2^{-i}$ and $|v_t - f_{t-1}(R_t)| \le 2^{-i}$, then let $y_t = v_t$. Otherwise we let $y_t = f_{t-1}(x_t)$. Finally, we define $f_t$ to be the function which linearly interpolates $\left\{(0, 0), (1, 0)\right\} \cup \left\{(x_s, y_s): s \le t\right\}$.

By definition, we have $f_t \in \mathcal{F}_{\infty}$ for each $t \ge 0$. We will prove next that for all $i, j$ we have $J[g_{i, j}] \le \frac{1}{4}$. The proof will use double induction, first on $i$ and then on $j$, and we will prove a slightly stronger statement. 

In order to prove that $J[g_{i, j}] \le \frac{1}{4}$ for all $i$ and $j$, we will prove that $J[f_{2^{i-1}-1}] \le \frac{\epsilon}{4} \sum_{k = 0}^{i-1} (1-\epsilon)^k$ for all $i \ge 1$. Note that this is equivalent to proving that $J[g_{i, 0}] \le \frac{\epsilon}{4} \sum_{k = 0}^{i-1} (1-\epsilon)^k$ for all $i \ge 1$. Clearly this is true for $i = 1$, which is the base case of the induction on $i$.

Fix some stage $i \ge 1$. We will assume that $J[f_{2^{i-1}-1}] \le \frac{\epsilon}{4} \sum_{k = 0}^{i-1} (1-\epsilon)^k$, and use this to prove that $J[f_{2^{i}-1}] \le \frac{\epsilon}{4} \sum_{k = 0}^{i} (1-\epsilon)^k$. In order to prove that $J[f_{2^{i}-1}] \le \frac{\epsilon}{4} \sum_{k = 0}^{i} (1-\epsilon)^k$, we will prove the stronger claim that $J[g_{i, j}] \le (\frac{\epsilon}{4} \sum_{k = 0}^{i-1} (1-\epsilon)^k )+ \frac{j \epsilon (1-\epsilon)^i }{2^{i+1}}$. This follows from the inductive hypothesis for $i$ and the definition of $g_{i, 0}$ when $j = 0$, which is the base case of the induction on $j$.

Fix some integer $j$ with $0 \le j \le 2^{i-1}-1$ and assume that $J[g_{i, j}] \le (\frac{\epsilon}{4} \sum_{k = 0}^{i-1} (1-\epsilon)^k )+ \frac{j \epsilon (1-\epsilon)^i }{2^{i+1}}$. By Lemma \ref{kl_jlem}, we have $J[g_{i, j+1}] = J[g_{i, j}]+ \frac{2(\frac{\sqrt{\epsilon}(1-\epsilon)^{\frac{i}{2}}}{2^{i+1}})^2}{2^{-i}} = J[g_{i, j}]+\frac{\epsilon (1-\epsilon)^i}{2^{i+1}}$.

By the inductive hypothesis for $j$, we obtain $J[g_{i, j+1}] \le (\frac{\epsilon}{4} \sum_{k = 0}^{i-1} (1-\epsilon)^k)+ \frac{j \epsilon (1-\epsilon)^i }{2^{i+1}} + \frac{\epsilon (1-\epsilon)^i}{2^{i+1}}$, which completes the inductive step for $j$. Substituting $j = 2^{i-1}$, we obtain $J[g_{i,2^{i-1}}] \le  \frac{\epsilon}{4} \sum_{k = 0}^{i} (1-\epsilon)^k$. Note that Lemma \ref{kl_jlem} implies that $J[f_{2^{i-1}-1+j}] \le J[g_{i, j}]$ for all $j = 0, \dots, 2^{i-1}$, so we obtain $J[f_{2^i-1}] \le  \frac{\epsilon}{4} \sum_{k = 0}^i (1-\epsilon)^k$ using $j = 2^{i-1}$, which completes the inductive step for $i$.

Since $J[g_{i, j}] \le (\frac{\epsilon}{4} \sum_{k = 0}^{i-1} (1-\epsilon)^k )+ \frac{j \epsilon (1-\epsilon)^i }{2^{i+1}}$, we obtain $J[g_{i, j}] \le (\frac{\epsilon}{4} \sum_{k = 0}^{i-1} (1-\epsilon)^k )+ \frac{\epsilon (1-\epsilon)^i }{4} = \frac{\epsilon}{4} \sum_{k = 0}^{i} (1-\epsilon)^k$ for all $j$ with $0 \le j \le 2^{i-1}$. Note that $\frac{\epsilon}{4} \sum_{k = 0}^{i} (1-\epsilon)^k <  \frac{\epsilon}{4} \sum_{k = 0}^{\infty} (1-\epsilon)^k = \frac{1}{4}$.

For each $i \ge 1$, we claim that $y_t = f_{t-1}(x_t)$ for at most half of the trials $t$ in stage $i$. For each trial $t$ with $y_t = f_{t-1}(x_t)$, note that the absolute value of the slope of $g_{i, t-2^{i-1}+1}$ must exceed $1$ in at least one of the intervals of length $2^{-i}$ on either side of $x_t$. If $y_t = f_{t-1}(x_t)$ for at least $b$ of the trials in stage $i$, then restricting to intervals of slope at least $1$ implies that $J[g_{i, 2^{i-1}}] \ge b 2^{-i}$.

Since $J[g_{i, 2^{i-1}}] \le \frac{1}{4}$, we must have $b \le 2^{i-2}$. Thus during stage $i$, there are at most $2^{i-2}$ trials $t$ with $y_t = f_{t-1}(x_t)$, which implies that there are at least $2^{i-2}$ trials with $y_t = v_t$. In each of those trials, $A$ was off by at least $\frac{\sqrt{\epsilon}(1-\epsilon)^{\frac{i}{2}}}{2^{i+1}}$, so the total $(1+\epsilon)$-error of $A$ after $i$ stages is at least $\sum_{k = 1}^i 2^{k-2} (\frac{\sqrt{\epsilon}(1-\epsilon)^{\frac{k}{2}}}{2^{k+1}})^{1+\epsilon}$. Thus

\begin{align*}
\opt_{1+\epsilon}(\mathcal{F}_{\infty}) \ge \\
\sum_{k = 1}^{\infty} 2^{k-2} (\frac{\sqrt{\epsilon}(1-\epsilon)^{\frac{k}{2}}}{2^{k+1}})^{1+\epsilon} = \\
\frac{\frac{1}{2}(\frac{\sqrt{\epsilon(1-\epsilon)}}{4})^{1+\epsilon}}{1-2 (\frac{\sqrt{1-\epsilon}}{2})^{1+\epsilon}} = \\
\Omega(\frac{(\epsilon(1-\epsilon))^{\frac{1+\epsilon}{2}}}{1-2 (\frac{\sqrt{1-\epsilon}}{2})^{1+\epsilon}}).
\end{align*}

Since $\epsilon^{\epsilon} = \Theta(1)$ and $(1-\epsilon)^{1+\epsilon} = \Theta(1)$ for $\epsilon \in (0, \frac{1}{2})$, we have $\opt_{1+\epsilon}(\mathcal{F}_{\infty}) = \Omega(\frac{\sqrt{\epsilon}}{1-2 (\frac{\sqrt{1-\epsilon}}{2})^{1+\epsilon}})$. Since $2^{\epsilon} = \Theta(1)$ for $\epsilon \in (0, \frac{1}{2})$, we have $\opt_{1+\epsilon}(\mathcal{F}_{\infty}) = \Omega(\frac{\sqrt{\epsilon}}{2^{\epsilon}-\sqrt{1-\epsilon}^{1+\epsilon}})$.

Note that $(1-\epsilon)^{\frac{1+\epsilon}{2}} \ge 1-\epsilon(1+\epsilon)$ for $\epsilon \in (0, \frac{1}{2})$. To check this, note that it is true when $\epsilon = 0$, and the derivative of $(1-\epsilon)^{\frac{1+\epsilon}{2}} - (1-\epsilon(1+\epsilon))$ is $2 \epsilon +1+ (1-\epsilon)^{\frac{1+\epsilon}{2}}(\frac{1}{2}\ln(1-\epsilon)-\frac{1}{2} -\frac{\epsilon}{1-\epsilon})$. For $\epsilon \in (0, \frac{1}{2})$, we have

\begin{align*}
2 \epsilon +1+ (1-\epsilon)^{\frac{1+\epsilon}{2}}(\frac{1}{2}\ln(1-\epsilon)-\frac{1}{2} -\frac{\epsilon}{1-\epsilon}) > \\
2 \epsilon +1+\frac{1}{2}\ln(1-\epsilon)-\frac{1}{2} -\frac{\epsilon}{1-\epsilon} > \\
2 \epsilon +1+\frac{1}{2}\ln(1-\frac{1}{2})-\frac{1}{2} -\frac{\epsilon}{1-\epsilon}  > \\
2 \epsilon +1-\frac{1}{2}\ln(2)-\frac{1}{2} -2\epsilon  >
0.
\end{align*}

Thus, $\opt_{1+\epsilon}(\mathcal{F}_{\infty}) = \Omega(\frac{\sqrt{\epsilon}}{2^{\epsilon}-1+\epsilon(1+\epsilon)})$. 

Also note that $2^{\epsilon} \le 1+\epsilon$ for $\epsilon \in (0, 1)$. There is equality at $\epsilon = 0$ and $\epsilon = 1$, and the derivative of $1+\epsilon - 2^{\epsilon}$ is $1-2^{\epsilon}\ln{2}$, which is positive for $\epsilon \in (0, \frac{-\ln{\ln{2}}}{\ln{2}})$ and negative for $\epsilon \in (\frac{-\ln{\ln{2}}}{\ln{2}}, 1)$. Thus $2^{\epsilon}-1+\epsilon(1+\epsilon) < 3\epsilon$, so $\opt_{1+\epsilon}(\mathcal{F}_{\infty}) = \Omega(\epsilon^{-\frac{1}{2}})$.
\end{proof}

The next corollary follows from Theorem \ref{lowerbound}, again using the fact that $\opt_p(\mathcal{F}_{\infty}) \le \opt_p(\mathcal{F}_r) \le \opt_p(\mathcal{F}_q)$ whenever $1 \le q \le r$.

\begin{cor}
If $\epsilon \in (0, 1)$, then $\opt_{1+\epsilon}(\mathcal{F}_q) = \Omega(\epsilon^{-\frac{1}{2}})$ for all $q \ge 1$, where the constant does not depend on $q$.
\end{cor}

\section{Discussion}

With the results in this paper, the value of $\opt_p(\mathcal{F}_q)$ is now known up to a constant factor for all $p, q \ge 1$ except for $(p, q)$ with $p \in (1, \infty)$ and $q \in (1, 2)$. It remains to investigate $\opt_p(\mathcal{F}_q)$ when $p \in (1, \infty)$ and $q \in (1, 2)$, and to narrow the constant gap between the upper and lower bounds for $\opt_{1+\epsilon}(\mathcal{F}_q) = \Theta(\epsilon^{-\frac{1}{2}})$ when $\epsilon \in (0, 1)$ and $q \in [2, \infty) \cup \left\{ \infty \right\}$.

A possible extension of this research is to investigate analogues of these problems for multivariable functions. Previous research on learning multivariable functions \cite{barron, hardle, haussler} has focused on expected loss rather than worst-case loss, using models where the inputs $x_i$ are determined by a probability distribution. 

In \cite{smooth2}, Long investigated a finite variant of $\opt_1(\mathcal{F}_q)$ for $q \ge 2$ that depends on the number of trials $m$. It seems interesting to extend this variant to $\opt_p(\mathcal{F}_q)$ for $p = 1+\epsilon$ with $0 < \epsilon < 1$ and $q \ge 1$, since $\opt_{1+\epsilon}(\mathcal{F}_q)$ can grow arbitrarily large as $\epsilon \rightarrow 0$.


\end{document}